\documentclass{article}

\usepackage{microtype}
\usepackage{graphicx}
\usepackage{subfigure}
\usepackage{booktabs} 

\usepackage[utf8]{inputenc} 
\usepackage[T1]{fontenc}    
\usepackage[USenglish]{babel}

\usepackage{amsfonts,mathtools,amssymb}       
\usepackage{amsthm}
\newtheorem*{rem}{Remark}
\newtheorem*{prop}{Proposition}

\usepackage{nicefrac}       
\usepackage{microtype}      
\usepackage{bm}
\usepackage{csquotes}
\usepackage{bbold}
\usepackage{enumitem}

\usepackage[accepted,nohyperref]{icml2020}

\usepackage{xr-hyper}

\usepackage[dvipsnames]{xcolor, colortbl}
\definecolor{Gray}{gray}{0.9}
\definecolor{arsenic}{rgb}{0.23, 0.27, 0.29}
\definecolor{azure}{rgb}{0.94, 1.0, 1.0}

\usepackage{tikz}
\usetikzlibrary{arrows.meta}
\usetikzlibrary{arrows,decorations.markings}

\newcommand*\circleds[1]{\tikz[baseline=(char.base)]{%
            \node[shape=circle,draw,inner sep=0.5pt] (char) {#1};}}
\newcommand{\snode}{\textsc{survNode}}

\icmltitlerunning{A General Method for Survival Analysis and Multi-State Modelling}

\usepackage{hyperref}

\begin{document}

\twocolumn[
\icmltitle{A General Framework for Survival Analysis and Multi-State Modelling}

\author{\name Stefan Groha\thanks{Equal contribution.} \email{stefanm\_groha@dfci.harvard.edu} \\
       \addr Dana--Farber Cancer Institute \\
       Harvard Medical School
       \AND
       \name Sebastian M Schmon$^*$\\
       \addr Department of Statistics \\
       University of Oxford
       \AND
       \name Alexander Gusev \\ 
       \addr Dana--Farber Cancer Institute \\ 
       Harvard Medical School}
\icmlsetsymbol{equal}{*}

\begin{icmlauthorlist}
\icmlauthor{Stefan Groha}{equal,df}
\icmlauthor{Sebastian M.~Schmon}{equal,imp}
\icmlauthor{Alexander Gusev}{df}
\end{icmlauthorlist}

\icmlaffiliation{df}{Dana--Farber Cancer Institute \\ 
       Harvard Medical School, USA}
\icmlaffiliation{imp}{Improbable, London, United Kingdom}

\icmlcorrespondingauthor{Stefan Groha}{stefanm\_groha@dfci.harvard.edu}

\icmlkeywords{Machine Learning, ICML}

\vskip 0.3in
]



\printAffiliationsAndNotice{\icmlEqualContribution} 

\begin{abstract}
Survival models are a popular tool for the analysis of time to event data with applications in medicine, engineering, economics, and many more. Advances like the Cox proportional hazard model have enabled researchers to better describe hazard rates for the occurrence of single fatal events, but are unable to accurately model competing events and transitions. Common phenomena are often better described through multiple states, for example: the progress of a disease modeled as healthy, sick and dead instead of healthy and dead, where the competing nature of death and disease has to be taken into account. Moreover, Cox models are limited by modeling assumptions, like proportionality of hazard rates and linear effects. Individual characteristics can vary significantly between observational units, like patients, resulting in idiosyncratic hazard rates and different disease trajectories. These considerations require flexible modeling assumptions.
To overcome these issues, we propose the use of neural ordinary differential equations as a flexible and general method for estimating multi-state survival models by directly solving the Kolmogorov forward equations. To quantify the uncertainty in the resulting individual cause-specific hazard rates, we further introduce a variational latent variable model and show that this enables meaningful clustering with respect to multi-state outcomes as well as interpretability regarding covariate values. We show that our model exhibits state-of-the-art performance on popular survival data sets and demonstrate its efficacy in a multi-state setting.
\end{abstract}

\section{Introduction}\label{sec:intro}
Time-to-event analysis is of fundamental importance in many fields where there is interest modelling event occurrence,  often in the presence of time-dependent missing outcomes (i.e. \enquote{censored} data). Examples include time-to-death analysis in medicine \citep{vigano2000survival}, failure of mechanical systems in engineering \citep{samaniego2007system} and financial risk \citep{dirick2017time}. If one event of interest is fatal, we speak of \emph{survival analysis}. 
For simplicity, most survival models only consider the binary case  where observations  transition from one non-fatal to a fatal state. 
The aim of many such models is to relate the arrival of events with observed characteristic information, e.g. model a patient's survival probability given their individual features. 
To date, the standard tool for survival analysis is the proportional hazards model, introduced in the seminal paper by \citet{cox1972regression}, which assumes a proportionality between the hazards for different values of the covariates of the model.

A first generalization of standard survival analysis considers multiple competing events, where all possible state transitions are fatal. For example, in the medical setting a patient can have multiple causes of death. For the incidence of these separate events, treating the other events as censored however leads to a bias due to misspecification of the at-risk population \citep{fine1999proportional} and the competing nature of the events has to be specifically modeled.

In recent years, with growing data availability and the advent of precision medicine, there has been increasing interest in a more refined modeling approach, taking into account multiple non-fatal states and more complicated relationships between all states 
\citep{rueda2019dynamics,gerstung2017precision,grinfeld2018classification,duffy1997,nicora2020continuous,Longini1989}. For example, in the case of acute myeloid leukemia, individualized genetic prediction based on a sophisticated multi-stage model was used to tailor personalized treatment within first complete remission \citep{gerstung2017precision}. In general, knowing which transitions and end-points are most likely to occur for a given patient enhances the clinician's ability for decision making.

The inclusion of covariates in common approaches for multi-state models usually requires making strong assumptions regarding the stochastic process and the dependence between model parameters and covariates.
We propose a general alternative approach, based on modeling the Kolomogorov forward equation of the underlying process using neural ordinary differential equations \citep{chen2018neural}. 
The use of neural networks provides considerably more model flexibility in comparison to previous approaches allowing the learning of expressive covariate relationships without placing any restrictive modeling assumption on the states. Directly modeling the underlying process gives us access to individual level cause-specific hazard rates and state occupation probabilities. A state augmentation akin to a memory process further enables us to move beyond the common Markov assumption in the state transition probabilities.
The method presented in this paper is, to the knowledge of the authors, the first neural network approach designed to explicitly handle multi-state survival models without using common simplifying assumptions and furthermore the first method to continuously model both time and probability distributions (or survival functions in the alive-dead survival case).
 
In summary, we demonstrate: 
\vspace{-1em}
\begin{itemize}
    \setlength\itemsep{-0.3em}
    \item a novel assumption free method for modeling survival outcomes that works with \emph{arbitrary numbers of states} with \emph{arbitrary topologies}, based on neural ordinary differential equations;
    \item state-of-the-art performance in survival analysis;
    \item superior performance in multi-state survival settings; 
    \item a variational training architecture for clustering multi-state survival outcomes with superior calibration of error intervals.
\end{itemize}

\section{Background and related work}

\subsection{Survival analysis}

Survival analysis is one of the simplest approaches for the study of time-to-event data. It categorizes the underlying states of interest as a dichotomous pair of a non-fatal and a fatal event, e.g. alive/dead for patients or functioning/failure for mechanical devices. Interest lies in the transition from the non-fatal to the fatal (absorbing) state.
Let $\boldsymbol{\tau}$ denote a random variable describing the time of the arrival of the fatal event. $\boldsymbol{\tau}$ can be flexibly modeled as the first jump of an inhomogeneous Poisson process with density function
\begin{equation*}
    f(t \mid \lambda) = \lambda(t) S(t),
\end{equation*}
where $\lambda$ denotes the \emph{hazard function} and $S(t) = \exp\big(-\int_0^t \lambda(s)\mathrm{d}s\big)$ is the \emph{survival function}.
In many cases, e.g. for patient data in clinical trials or for observational data, some of the participants will drop out at an earlier stage than the time of conclusion of the study. This gives an ambiguous meaning to the observed time points, $t$, which is a fatal-event if there is no censoring $(\delta=1)$ or a drop-out $(\delta=0)$. In the latter case the only information available is that $\boldsymbol{\tau} > t$ which has probability $\mathbb{P}(\boldsymbol{\tau} > t) = S(t)$. This is a case of \emph{right-censoring}.
Assuming independence of the censoring process the likelihood contribution of an individual $i$ is
\begin{align}
    \mathcal{L}_i =\lambda(t_i)^{\delta_i}S(t_i).
\end{align}

The most widely used tool to obtain the influence of covariates on the survival function $S(t)$ is the Cox proportional hazards model. This method is a semi-parametric method for the hazard function $\lambda(t)$, which is modeled as $\lambda(t) = \lambda_0(t) \exp\left(\bm{\beta}^T \bm{x}\right)$, where ${\bm \beta}$ are coefficients for the covariates $x$ and $\lambda_0(t)$ is a baseline hazard directly estimated from the data. Both the linear nature of the model as well as the proportional hazards assumption are often violated in practice.

Many extensions of the Cox proportional hazards model have been proposed, aiming to relax one or both of those assumptions. This includes models using the Cox model structure, but extending it to non-linear features or non-proportional hazards, e.g. by modeling $\lambda(t)=\lambda_0(t) \exp\left(f_\theta(x)\right)$ with $f_\theta(x)$ being a deep neural network, or $\lambda(t)=\lambda_0(t) \exp\left(f_\theta(x,t)\right)$ (continuous time models) \citep{katzman2018deepsurv,kvamme2019time};  approaches using MLPs \citep{lee2018deephit} or recurrent neural networks \citep{giunchiglia2018rnn,ren2019deep} for every time step (discrete time models); Gaussian Process models \citep{alaa2017deep,fernandez2016gaussian} or generative adversarial networks (GANs) \citep{chapfuwa2018adversarial}.

\subsection{A progressive three-state survival model}
The aim of multi-state models is a more granular analysis of time-to-event phenomena, where common binary outcomes (i.e. health/death) can not adequately describe real observations. A simple extension of a traditional binary survival model observing the time to a fatal event is the addition of an intermediate state, illness, which could denote the appearance of symptoms or, more generally, some non-fatal disease progression event. Such models and their state-space can be described by a directed graph as shown in \autoref{fig:ill_death}.
Processes that evolve continuously over time where observational units (like patients) move between states are referred to as continuous time, finite state space Markov processes.  
Such a Markov process is completely characterized by the (matrix of) state transition probabilities for all tuples of states $(i, j)$ and all tuples of time points $(s, t)$
\begin{equation*}
    P_{ij}(s,t) = \mathbb{P}(Y(t)=j \mid Y(s)=i),
\end{equation*}
where $Y(s)$ denotes the state of an individual at time $s$.

\begin{figure}[t]
\centering
\subfigure[Competing risks model.]{\begin{tikzpicture}[arsenic, scale=0.45]
      \draw[thick, fill=arsenic, rounded corners=2pt] (-2.7,-0.7) rectangle (2.7,0.7);
      \node[text=azure] at (0,0) {\circleds{\textsf{1}} \textsf{Health}};
      \draw[thick, fill=arsenic, rounded corners=2pt] (5,1.3) rectangle (9.9,2.7);
      \node[text=azure] at (7.3,2) {\circleds{\textsf{2}} \textsf{Cause A}};
      \draw[thick, fill=arsenic, rounded corners=2pt] (5,-1.3) rectangle (9.9,-2.7);
      \node[text=azure] at (7.3,-2) {\circleds{\textsf{3}} \textsf{Cause B}};
      \draw [->,line width=2pt,>=latex] (0, 0.5) -- (5, 2);
      \draw [->,line width=2pt,>=latex] (0, -0.5) -- (5, -2);
      \node[rotate= 15] at (3.1,2) {$\lambda_{12}(t)$};
      \node[rotate=-15] at (3.1,-2) {$\lambda_{13}(t)$};
    \end{tikzpicture}\label{fig:comp_hazard}}
\subfigure[General multi-state model.] {\begin{tikzpicture}[arsenic, scale=0.45]
      \draw[thick, fill=arsenic, rounded corners=2pt] (-2.7,-0.7) rectangle (2.7,0.7);
      \node[text=azure] at (0,0) {\circleds{\textsf{1}} \textsf{Health}};
      \draw[thick, fill=arsenic, rounded corners=2pt] (5,1.3) rectangle (9.9,2.7);
      \node[text=azure] at (7.2,2) {\circleds{\textsf{2}} \textsf{Illness}};
      \draw[thick, fill=arsenic, rounded corners=2pt] (5,-1.3) rectangle (9.9,-2.7);
      \node[text=azure] at (7.2,-2) {\circleds{\textsf{3}} \textsf{Death}};
      \draw [->,line width=2pt,>=latex] (0, 0.5) -- (5, 2);
      \draw [->,line width=2pt,>=latex] (0, -0.5) -- (5, -2);
      \draw [->,line width=2pt,>=latex] (7.45, 1.3) -- (7.45, -1.3);
      \node[rotate= 15] at (3.1,2) {$\lambda_{12}(t)$};
      \node[rotate=-15] at (3.1,-2) {$\lambda_{13}(t)$};
      \node at (8.8,0) {$\lambda_{23}(t)$};
    \end{tikzpicture}\label{fig:ill_death}}
\caption{Example graphs corresponding to the 2-state competing risks model (a) and the illness-death model (b), a popular multi-state model. Competing risks models are a special case of multi-state models that only have one non-absorbing state, whereas in general multi-state models can have arbitrary, even cyclical connections.
}
\end{figure}
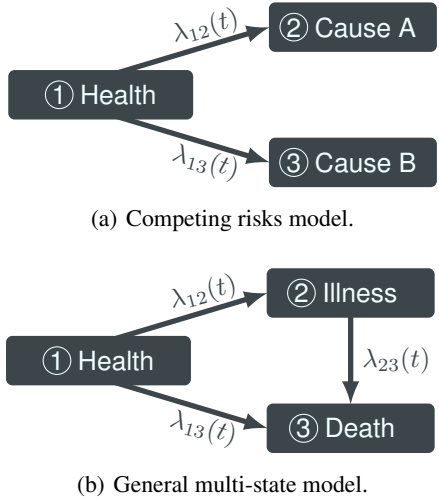

Describing the transition probabilities, and hence the likelihood, with a model that allows for flexible use of covariates, while allowing non-homogeneous state evolution is challenging.
The standard tool is a Markov multi-state model, where a Cox proportional hazards model is applied to each transition separately. 
The transition probabilities are estimated by assuming a Markov model for the transition through states \citep{mstate}. 
This framework has the disadvantages of the Cox proportional hazard model at each transition and additionally a Markov assumption for each state, together with the assumption that event times for different events are independent of each other, which is rarely given in practice. 


The conceptually more appealing approach of modeling the events as a Markov jump process, solving the Kolmogorov forward equation was introduced in \citep{titman2011kfespline}. However, the proposed B-spline basis for the hazard function does not generalize well to inclusion of covariates, as a separate Kolmogorov forward equation has to be fit for every realization of the covariates.
Recently, generalizations to the special case of competing risks models using Gaussian Processes \citep{alaa2017deep} and deep neural networks \citep{lee2018deephit} were proposed, however we are not aware of any literature considering an extension of such flexible methods to the setting of general multi-state models.

\section{Multi-state survival models}
\label{sec:multi_state}
Mathematically, multi-state models are defined as a continuous time stochastic process $\{Y(t); 0 \leqslant t \leqslant T\}$ taking values in a finite state space $\mathsf{Y}=\{1,...,S\}$ over known time horizon $T > 0$. Such processes are often called (Markov) jump processes. In the following we will describe the likelihood function and its relation to the Kolmogorov forward equations \citep{kolmogoroff1931analytischen, feller1949}.

\subsection{Markov Jump Processes and the Kolmogorov Forward Equations}
If the stochastic process $\{Y(t); 0 \leqslant t \leqslant T\}$ is Markovian it can be fully characterized by its transition kernel from time $s$ to $t$, denoted $P(s,t) \in \mathbb{R}^{S\times S}$ with elements
\begin{equation*}
    P_{ij}(s,t) = \mathbb{P}(Y(t)=j \mid Y(s)=i), 
\end{equation*}
$i=1,\ldots, S, j=1, \ldots, S$.
As shown by \citet{kolmogoroff1931analytischen} such transition kernels follow a set of differential equations 
\begin{align}
    \frac{\mathrm{d}P_{ij}(s,t)}{\mathrm{d}t} =  \sum\nolimits_{k} P_{ik}(s,t) \lambda_{kj}(t), 
\end{align}
$i = 1,\ldots, S, \quad j = 1, \ldots S$, called the \emph{Kolmogorov forward equations}.
\subsection{Multi-state likelihood with known transition times}

For each individual we will observe the process $Y(t)$ in the form of discrete jumps over the relevant time interval $[0, T]$.

In this setting, a single observation consists of a set of $m$ time-indexed states ${y(t_1), \ldots, y(t_m)}$. The likelihood is given by 
\begin{align*}
    & P\left(y(t_1), \ldots, y(t_m); \theta\right) \\
    & \quad =  P\left( y(t_1) \right)\prod\nolimits_{j=2}^{m} T\left(y(t_{j-1}), y(t_{j}) \right),
\end{align*}
where the transition probability, $T$ is
\begin{equation}\label{eq:exact_time_likelihood}
\begin{aligned}
    & T\left(y(s), y(t) \right) = \\
    & \,\,\underbrace{P_{y(s)y(s)}\left(s, t; \theta \right)}_{\text{stay at $y(s)$ from $s$ to $t$}} \quad \times \underbrace{\lambda_{y(s)y(t)}(t \mid \theta).}_{\text{jump from $y(s)$ to $y(t)$ \emph{at} time $t$}}   
\end{aligned}
\end{equation}
The value $\theta$ denotes all free model parameters and $P\left( y(t_1) \right)$ the probability to be in the initial state. To ensure the likelihood is well-defined for $m=1$, we define an empty product as $\prod_{j=2}^1 = 1$. 
The full likelihood for a set of $n$ observations is thus given by
\begin{equation*}
    \mathcal{L}(\theta ; \mathcal{Y}) = \prod\nolimits_{i = 1}^n 
    P\left(y_i(t^i_1), \ldots, y_i(t^i_{m_i}); \theta\right),
\end{equation*}
with $\mathcal{Y} = \{y_1, \ldots, y_n\}, y_i = \{y_i(t^i_1), \ldots, y_i(t^i_{m_i})\}, i = 1, \ldots, n$.
Under the Markov assumption the evolution of the transition probabilities in the likelihood is governed by the Kolmogorov forward equation
\begin{align}
    \frac{\mathrm{d}P_{ij}(s,t)}{\mathrm{d}t} =  \sum\nolimits_{k} P_{ik}(s,t) \lambda_{kj}(t), 
\end{align}
$i = 1,\ldots, S, j = 1, \ldots S$, where the Markov property is evident by the fact that the instantaneous transition rates $\lambda_{ij}(t)$ are only dependent on the time $t$. 

\subsection{Multi-state likelihood with unknown transition times (interval censoring)}
In the previous section, we have assumed that the exact time of the transitions are known, leading to the likelihood transition as shown in \eqref{eq:exact_time_likelihood}.
This is not always the case, but instead it might only be known that the transition happened between two time points and the likelihood needs to be adjusted accordingly. Instead of computing \eqref{eq:exact_time_likelihood}, we then need to substitute
$T\left(y(s), y(t) \right) = P_{y(s)y(t)}\left(s, t; \theta \right).$

\subsection{Right-censoring}
\label{sec:rightcensoring}
As alluded to earlier, censoring (to account for missing observations) is common in survival models and requires an adjustment of the likelihood function. 
Assuming independence of the censoring process, we observe $\{x_i,y_i,\delta_i; j=1,\ldots,m_i, i=1,\dots,n\}$, where $x_i$ are individual covariates or regressors, $m_i$ is the number of transitions the individual $i$ is going through and $y_i$ are as above or the state at time of last contact (censoring time). Censoring is indicated by $\delta_i = 0$ whereas we write $\delta_i = 1$ if the event is observed. The corresponding likelihood can then be written as 
\begin{align*}
    & \mathcal{L}(\theta; \mathcal{Y}) =  \prod\nolimits_{i=1}^n P\left( y_i(t^i_1)  \right) \times \\
    & \prod\nolimits_{j=2}^{m_{i-1}} P_{y_i(t^i_{j-1})y_i(t^i_{j-1})}\left(t^i_{j-1}, t^i_{j}; \theta \right) \lambda_{y_i(t^i_{j-1})y_i(t^i_j)}(t^i_j \mid \theta) \\ &\qquad\times P_{y_i(t^i_{m_i-1})y_i(t^i_{m_i-1})}\left(t^i_{m_i-1}, t^i_{m_i}; \theta \right) \\
    &\qquad \times \left(\lambda_{y_i(t^i_{m_i-1})y_i(t^i_{m_i})}(t^i_{m_i} \mid \theta)\right)^{\delta_i}.
\end{align*}

\begin{rem}[Left-truncation]
We note that the above likelihood also allows for possible left-truncation, where a patient is added at a later time, but is known to be in a certain state up until this point, for example to control for immortal time bias.
\end{rem}

\section{\textsc{survNode}: neural ODEs for multi-state modelling}

\subsection{Model definition}
We define our model by parameterizing the Kolmogorov forward equations directly. This is achieved by modeling the instantaneous transition rate matrix $\boldsymbol{Q}(t)$ with a neural network. Ensuring the conservation of probability requires that the elements $(\bm{Q})_{ik} = \lambda_{ik}$ of the transition rate matrix $\bm{Q}$ need to fulfill
\begin{equation*}
    \sum\nolimits_k \lambda_{ik}(t) = 0.
\end{equation*}
This restriction can be implemented by modelling $\lambda_{ij}(t), i\neq j$ through the neural network and set $\lambda_{ii}(t) = - \sum_{k \neq i} \lambda_{ik}(t)$. 
As we need the transition rates to be larger than $0$, we use a \texttt{softplus} activation on the last layer of the network.

To incorporate the covariates we use the following approach. Instead of only modeling the Markovian transition rate $\boldsymbol{Q}(t)$, we incorporate the history of the evolution and the covariate state of individual $i$ as $\boldsymbol{Q}(t,\mathcal{H}(t))$. For this we introduce auxiliary memory states $m(t)$, governed by the differential equation 
\begin{align*}
    \frac{\mathrm{d}m_i}{dt} = M_i(t,\bm{P}(t),\bm{m}(t)).
\end{align*}
The initial conditions are encoded by the covariates of the patient $m(0) = f(x)$, where $f$ is given by a neural net. We can then obtain the system of coupled ODEs
\begin{align*}
    \frac{\mathrm{d}P_{ij}(0,t)}{\mathrm{d}t} &=  \sum\nolimits_k P_{ik}(0,t) \lambda_{kj}(t,\bm{P}(0,t),\bm{m}(t),x) \\
    \frac{\mathrm{d}P_{ij}(s,0)}{\mathrm{d}s} &= - \sum\nolimits_{k} \lambda_{ik}(s,\bm{P}(0,s),\bm{m}(s),x) P_{kj}(s,0) \\
    \frac{\mathrm{d}m_i}{\mathrm{d}t} &= \textstyle M_i(t,\bm{P}(t),\bm{m}(t),x).
\end{align*}
where the second line is the Kolmogorov backward equation. 

\begin{table*}[ht]
  \caption{Benchmark of \snode. The results for the other models are taken from \citep{kvamme2019time}. Note that the results of competing methods are more hyperparameter optimized. Higher concordance and lower ibs and ibll are better. The best result is highlighted in bold.}
  \label{tab_1}
  \centering
    \begin{tabular}{c|c|c|c|c|c|c}
    Model  & {\small\textsc{metabric}} & {\small\textsc{metabric}} & {\small\textsc{metabric}} & {\small\textsc{support}} & {\small\textsc{support}} & {\small\textsc{support}}\\
    & c & ibs & ibll & c & ibs & ibll\\ 
    \hline
    Cox-PH\citep{cox1972regression}& 0.628 & 0.183 &-0.538 & 0.598 & 0.217 & -0.623 \\
    DeepSurv \citep{katzman2018deepsurv} & 0.636 & 0.176 &-0.532& 0.611 & 0.214 & -0.619\\
    Cox-Time \citep{kvamme2019time} & 0.662 & 0.172 &-0.515 & 0.629 & 0.212 & -0.613\\
    DeepHit \citep{lee2018deephit} & {\bf 0.675} & 0.184 &-0.539& {\bf 0.642} & 0.223 & -0.637\\
    RSF \citep{ishwaran2008random} & 0.649 & 0.175 & -0.515 & 0.634 & 0.212 & -0.610\\
    \rowcolor{Gray}
    \snode  & $0.667$ & \bf{0.157} & \bf{-0.477} & 0.622 & {\bf 0.198} & {\bf -0.580}
    \end{tabular}
\end{table*}

\subsection{Implementation details}
Using that $P(s,0)=P^{-1}(0,s)$ and therefore $\sum_k P_{ik}(s,0) P_{kj}(0,t) = P_{ij}(s,t)$, we obtain $P_{ij}(s,t)$ at any $s$ and $t$.

We model both $\lambda_{ij}(t,\bm{P}(0,t),\bm{m}(t),x)$ and $M_i(t,\bm{P}(0,t),\bm{m}(t),x)$ with one neural network $g(t,\bm{P}(0,t),\bm{m}(t),x)$, where the first $q$ (number of non-zero off-diagonal elements of $\bm{Q}$) outputs of the last layer are passed through a \texttt{softplus} non-linearity. This generalizes the approach in \citep{chen2018neural} and shares some conceptual ideas with \citep{jia2019neural}. Another interpretation of the memory states is the augmentation of the neural ODE with additional states as seen in \citep{dupont2019augmented}. The algorithm is shown in the appendix. 

Following \citep{massaroli2020dissecting}, we furthermore add an $L_2$ loss term for the time evolved memory states at the maximum time of the training batch, which can be seen as some modification of minimizing a Lyapunov exponent such that comparable initial values produce comparable survival.

With this model, we also have direct access to the hazard rate (the instantaneous risk for a given transition) over time. By predicting the hazard rates for the possible realizations of e.g. a binary feature over time and taking the ratio, we can derive a personalized predictive score for the influence of that feature on the transition rates between states. Such time-dependent hazard ratios are critical for predicting treatments or identifying biomarkers in a clinical setting.

Due to the encoding of the covariates into the initial values of the memory states this model can naturally extend to include features based on longitudinal data, text data or imaging data by encoding the initial values with recurrent neural network layers, natural language processing layers or convolutional layers and training those at the same time. Time dependent covariates can similarly be incorporated with recurrent neural networks before every new measurement of the time dependent feature.
The model is implemented in PyTorch \citep{pytorch} using the torchdiffeq \citep{chen2018neural} package.

\section{Variational \snode: modeling uncertainty}

To obtain a quantification of model uncertainty, we further extend the model to a variational setting by introducing latent variables. Instead of maximum likelihood estimation, the objective will be the variational free energy or evidence lower bound \textsc{elbo}.
The variational model assumes the existence of a latent state $z$, which replaces the role of the memory state $\bm{m}(0)$ above, such that $\mathcal{L}(\theta; \mathcal{Y}, \boldsymbol{z})$ does not depend on the covariates $\boldsymbol{x}$ given $\boldsymbol{z}$.
The objective is then
\begin{align*}
    &\textsc{elbo}(\theta, \mathcal{Y}) = \\ &\mathbb{E}_{q(\boldsymbol{z} \mid t, \boldsymbol{x})}\big[\log \mathcal{L}(\theta; \mathcal{Y}, \boldsymbol{z}) \big] - \mathcal{D}_{KL}\big(q(\boldsymbol{z} \mid t, \boldsymbol{x}) \,\|\, p(\boldsymbol{z} \mid \boldsymbol{x})\big) 
\end{align*}

where we model the variational distribution $q(\bm{z}|t,\bm{x})$ and the prior $p(\bm{z}|\bm{x})$ as
\begin{align*}
    q(\bm{z}|t,\bm{x}) &= \mathcal{N}\big(\mu_q(x,t),\textrm{diag}(\sigma_q^2(x,t))\big) \\
    p(\bm{z}|\bm{x}) &= \mathcal{N}\big(\mu_p(x),\textrm{diag}(\sigma_p^2(x))\big)
\end{align*}
with neural networks for $\mu_q$, $\mu_p$, $\sigma_q$, and $\sigma_p$, encoding the covariates into the latent space. 

For prediction, we obtain realizations of the transition matrix $P_{ij}(0,t)$ by repeated sampling from the prior and taking the mean as well as the $95\%$ credible interval
\begin{align*}
    z(0) &\sim p(z\mid x), \\
    P_{ij}(0,t\mid z) &= \texttt{ODEsolve}((\mathbb{1},\mathbb{1},z(0)),\texttt{KFE\_KBE},(0,t)).
\end{align*}
Details can be found in the supplementary material.

\section{Experiments}

\subsection{Survival: benchmark of model}
To benchmark our proposed model against various survival frameworks, we examine the performance of \snode{} on the \textsc{metabric} breast cancer data set \citep{curtis2012genomic, pereira2016somatic}, as well as the \textsc{support} data set \citep{knaus1995support}. 

In order to measure the performance of our model we score our survival prediction using the following standard measures \citep[see][for precise definitions]{kvamme2019time}:
\vspace{-1em}
\begin{itemize}
    \setlength\itemsep{-0.3em}
    \item \emph{Concordance}. The concordance or c-index is the relative fraction of \emph{concordant} sample pairs, where a pair of samples is \emph{concordant} when prediction and observation have the same order. The c-index therefore measures a models discriminatory power. 
    \item \emph{(Integrated) Brier score}. The Brier score at time $t$ measures the (squared) difference between the forecasted outcome probability and the actually observed value providing a measure of calibration and discrimination. The integrated Brier (ibs) summarizes the Brier score over all time points.
    
    \item \emph{(Integrated) Binomial log-likelihood} The binomial log-likelihood is based on the binary cross entropy for the probability of a patient to still be alive at time $t$. As the Brier score, it can subsequently be integrated over time and similarly measures both discrimination and calibration of a method.

\end{itemize}
When censoring is present, the formulas have to be adjusted accordingly, see e.g. \citep{spitoni2018prediction} for the Brier score in the multi-state case. 
The Brier score, in contrast to the concordance, is a \emph{proper scoring rule}, meaning that predicting the true outcome probabilities provides maximal reward or minimal loss, thus making it the preferable measure in our applications.

We present the benchmark concordance (c)\citep{antolini2005time}, integrated Brier score (ibs)\citep{brier1951verification}, as well as the integrated binomial log-likelihood estimator (ibll) with five-fold cross validation in \autoref{tab_1}.

\begin{figure}[th]
    \centering
    \subfigure[\snode{}]{
    \includegraphics[width=0.4\textwidth]{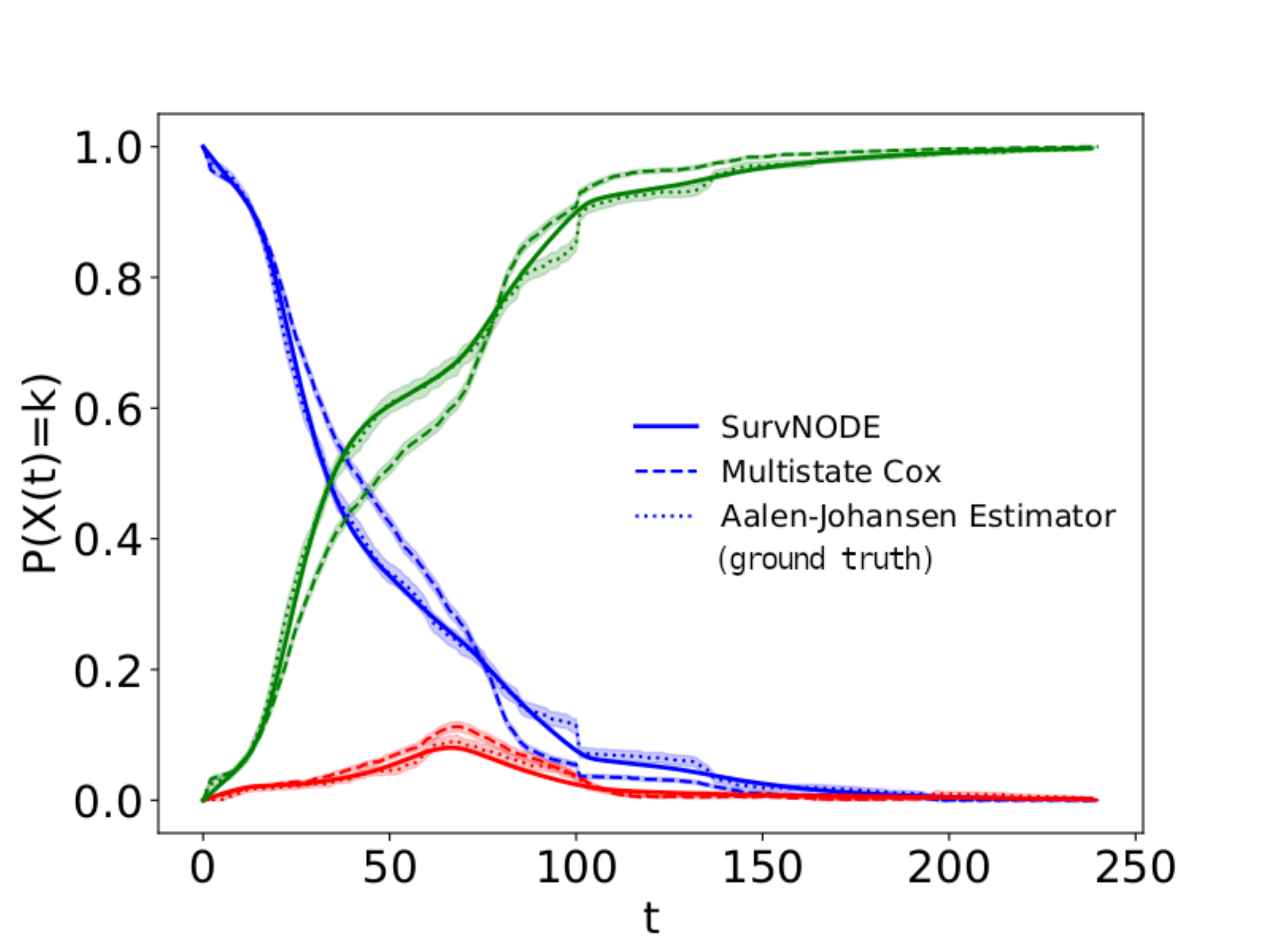} \label{fig:multistate_coxcompare}}
    \subfigure[Variational \snode{}]{
    \includegraphics[width=0.4\textwidth]{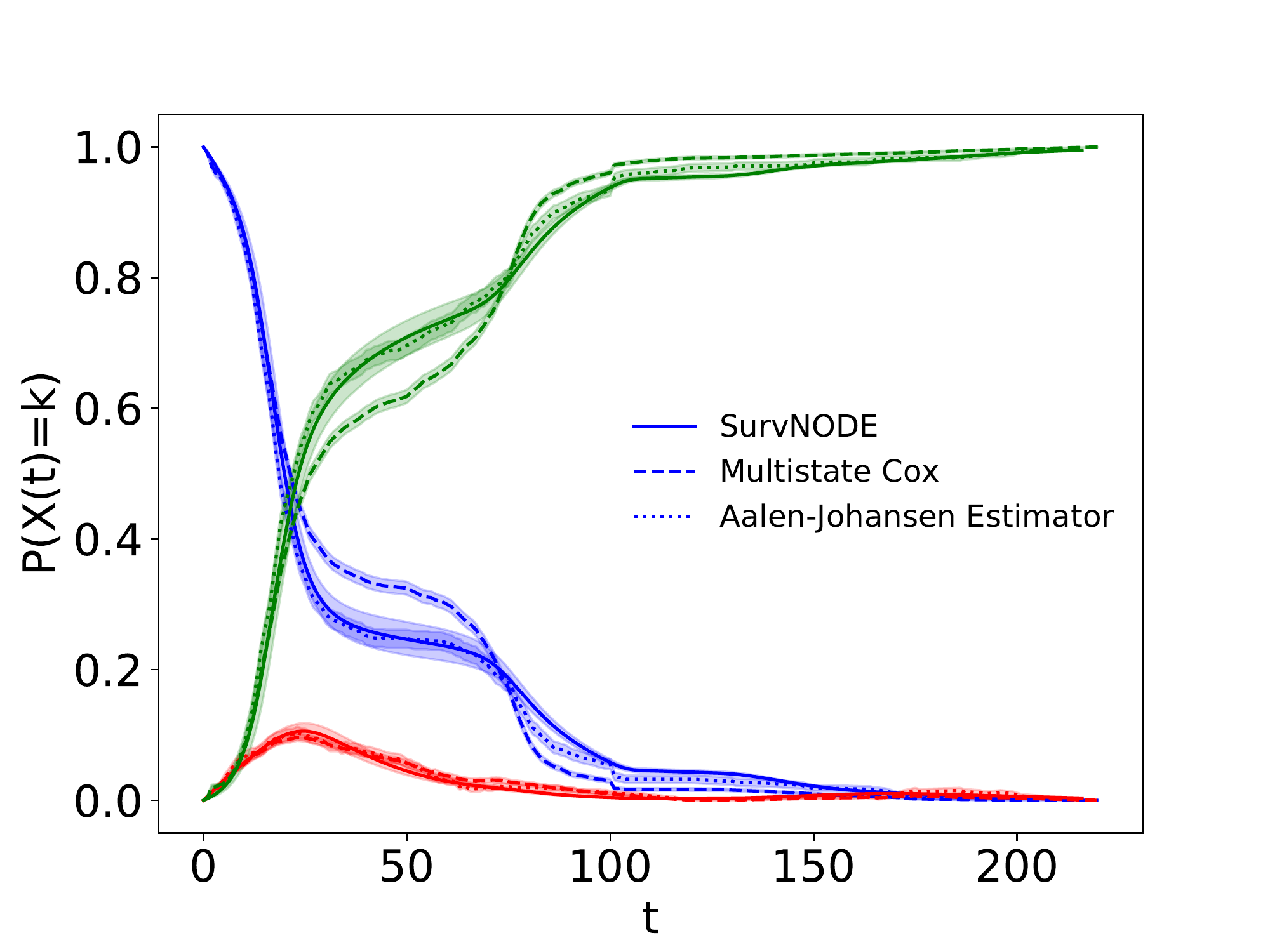}
    \label{fig:multistate_coxcompare_bayes}
    }
    \caption{Plot of probabilities for being in the different states of the illness-death model, indicated by the different colors. Blue corresponds to \enquote{Health}, red to \enquote{Illness} and green to \enquote{Death}. We compare to the cause-specific multi-state Cox model in addition to the non-parametric Aalen--Johansen estimator as ground truth. \emph{In all cases \snode{} reliably predicts the population level probabilities where the multi-state Cox model shows significant deviations.}}
   
\end{figure}

As can be seen in \autoref{tab_1}, our method outperforms all competitors in terms of the proper scoring rules (ibs, ibll) while attaining state-of-the-art discriminative performance as measured by the concordance index. While DeepHit slightly outperforms our model in terms of concordance, this comes at the price of a significantly worse Brier score, even compared to a Cox model.  An illustration of this is shown in the supplemental material. Although concordance is a common figure of merit, for clinical applications of predictive models for precision medicine, it may be as or more important to have a well calibrated probability for the event to provide the clinician with unbiased decision support~\citep{graf1999ibsvsconc,hand1997construction,hilden1978probs}.

For diagnostic tests in a clinical setting, for instance, a low integrated Brier score corresponds to a better predictive value of the diagnosis, meaning the probabilities of a positive or negative diagnosis are closer to the real underlying probabilities. A higher concordance, on the other hand, will give a better classification into positive diagnosis or negative diagnosis \citep{graf1999ibsvsconc}. These metrics are distinct in their diagnostic ability of the model prediction and there can be tradeoffs between maximising concordance vs maximizing Brier score.

\subsection{Multi-state survival}
A major advantage of \snode{} is that is applies equally to an arbitrary multi-state survival model distinguishing it from classical (binary) survival models. 

To show the efficacy of our model in the multi-state setting and to visualize the advantage of our model over the only other existing method for the general multi-state setting\footnote{For some simple extension of survival models, for example, the competing risks model (\autoref{fig:comp_hazard}) other methods exist. A comparison to such a model is shown in the appendix.}, we investigate the models performance on an aggregate and individual level. \vspace{-10pt}

\paragraph*{Population level comparison}
To judge the overall population level performance of \snode{} we compare with a non-parametric estimator for an illness-death model (see \autoref{fig:ill_death}). 
In the multi-state setting a population mean for the probabilities in each state can be obtained with the Aalen--Johansen estimator \citep{aalen1978empirical}. 
As a baseline, we simulate a data set with proportional hazards violation using the \textbf{\texttt{coxed}} \textsf{R} package \citep{harden_kropko_2019}(see supplementary material).
We compare our model with the standard tool in the multi-state survival literature, which is fitting a Cox proportional hazard model to each transition, treating the other events as censored. 
Importantly, this assumes independence between the occurring events, which while true for the simulated data set is often not the case in real world scenarios. We use the \textsf{R} package \textbf{\texttt{mstate}} \citep{mstate} to obtain the state probabilities at each time. 

The comparison can be seen in \autoref{fig:multistate_coxcompare}, where we plot the probabilities for the occupation of every state over time for both models together with the ground truth, estimated by the Aalen--Johansen estimator. We see a clear advantage of our model over the cause-specific Cox model, being more accurate for all transitions at all times than the multi-state Cox model.

\begin{figure}[h]
    \centering
    \includegraphics[width=0.4\textwidth]{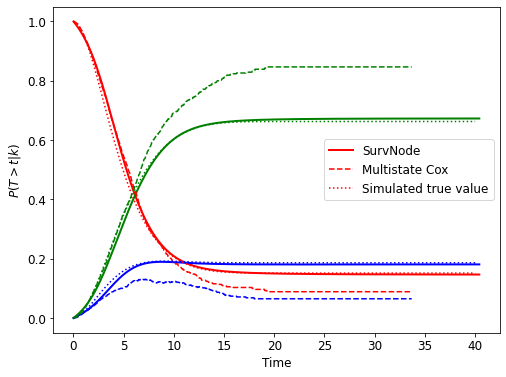}
    \vspace{-10pt}
    \caption{Probability distribution to be in each of the three states of the illness death model for a random individual patient. The dotted lines are the real underlying probability distributions obtained from the simulated hazard rates, which in contrast to the multi-state Cox model are almost perfectly recovered by \snode{}.}
    \label{fig:multi_sim_probs}
\end{figure}

\begin{figure}[h]
    \centering
    \includegraphics[width=0.4\textwidth]{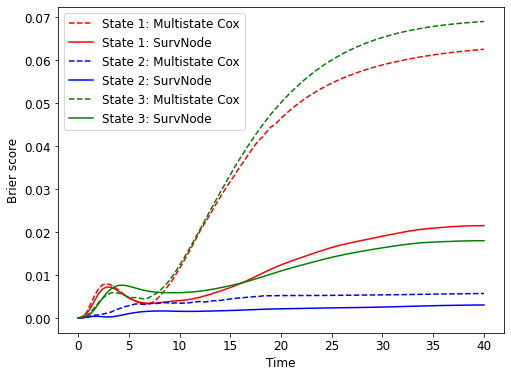}
    \vspace{-10pt}
    \caption{Brier score over time of \snode{} and the multi-state Cox model. We clearly see a superior calibration of \snode{} as compared to the multi-state Cox model.}
    \label{fig:multi_sim_brier_exact}
\end{figure}

\paragraph{Precision medicine: accurate individual level prediction}
With highly flexible neural network survival models such as \snode{}, we hope to provide a more accurate individual level prediction.  
We test the individual level performance by simulating another illness-death data set directly from a Markov jump process with multiple covariates (see appendix), again assuming a proportional hazards violation. We quantify the calibration of both \snode{}, as well as the multi-state Cox model on an individual patient level. 
As the underlying true probability distribution for each patient is known (a probability distribution for a random individual patient is plotted in \autoref{fig:multi_sim_probs}), we can directly calculate the Brier score as the squared difference in real underlying and predicted probabilities to be in each state over time. The Brier score is shown in \autoref{fig:multi_sim_brier_exact} and we  see that \snode{} provides superior estimates of the probability distributions with the predicted probabilities of the multi-state Cox model on average having a severely worse calibration than \snode{} at most times.  

\paragraph*{Real data example}
Lastly, we compare the two models using a multi-state generalization of the Brier score \cite{spitoni2018prediction} on data by the European Society for Blood and Marrow Transplantation (ESBMT) \cite{fiocco2008reduced}. The multi-state model is a more complicated six-state model and is shown in \autoref{fig:ebmt}.

The Brier scores for the state probabilities of the six states for both \snode{}, as well as the multi-state Cox model, are shown in \autoref{fig:brier_real}. In this case we see a comparable Brier score, and therefore a matching prediction. This is likely due to the absence of  interactions between the covariates in the data and negligible proportional hazards and Markov violations. This demonstrates, however, that the increased flexibility  of \snode{} does not hurt performance when the assumptions of the simpler Cox model are met. In contrast, we are able to produce stable predictive results compared with the current standard tool in multi-state survival analysis.

\begin{figure}[h]
    \centering
    \includegraphics[width=0.9\linewidth]{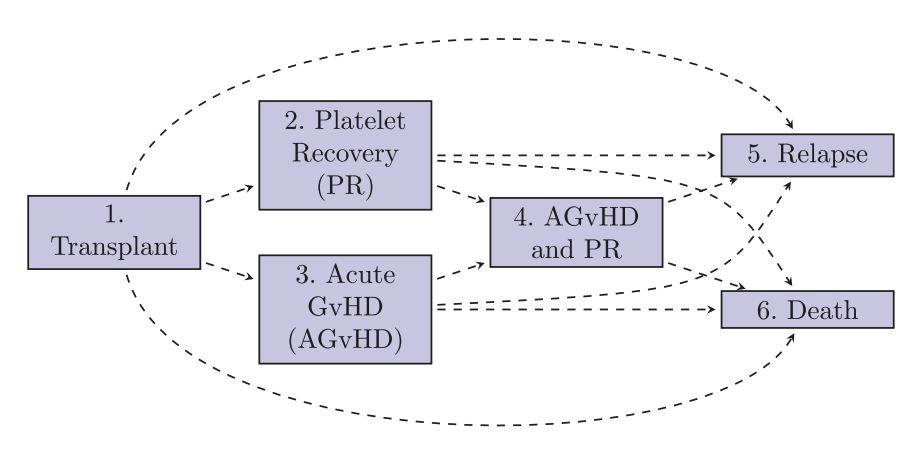}
    \vspace{-15pt}
    \caption{Diagram for possible transitions in the data by the ESBMT. The hazard rates for each of the 12 transitions have to be modeled.}
    \label{fig:ebmt}
\end{figure}

\begin{figure}[h]
    \centering
    \includegraphics[width=0.8\linewidth]{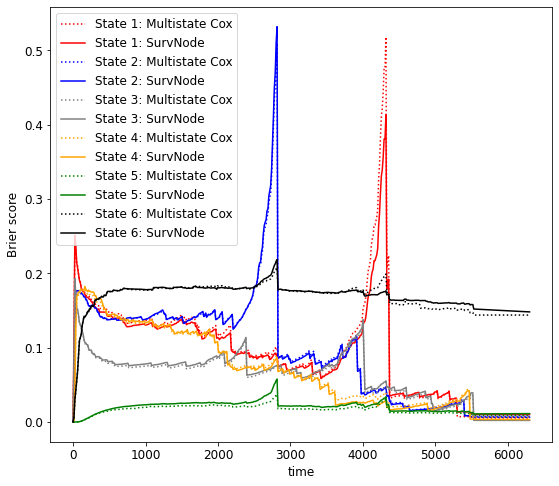}
    \vspace{-15pt}
    \caption{Brier score over time for both multi-state Cox model as well \snode{} for each of the six states of the multi-state model.}
    \label{fig:brier_real}
\end{figure}

\begin{table*}[t]
\centering
\caption{Comparison of \snode and the Cox proportional hazard model in terms of calibration and concordance.}
\label{tab_3}
\begin{tabular}{c|c|c|c}
Model & calibration & concordance \\ \hline
Cox proportional hazards model \citep{cox1972regression}& $0.250\pm0.015$ & $0.6244\pm0.0064$  \\
\rowcolor{Gray}
\snode{} (this paper) & $\mathbf{0.749\pm0.071}$  & $\mathbf{0.6396\pm0.0097}$
\end{tabular}
\vspace{-10pt}
\end{table*}

\subsection{Variational \snode{}}

\paragraph{Benchmark}
As a first step of analysis of the latent multi-state survival model, we benchmark the model against other models for the special case of survival analysis on the \textsc{metabric} data set. Without any hyper-parameter tuning and ad-hoc parameter choice, we obtain a concordance of $c=0.646$, integrated Brier score of $0.170$ and integrated binomial log-likelihood of $0.503$. As such the variational \snode{} is better calibrated than all other available models with competitive discrimination performance (not including plain \snode{}). We can visualize the prediction and confidence interval by again comparing to the Aalen--Johansen estimator in the simulated illness-death model in \autoref{fig:multistate_coxcompare_bayes}. For this we have trained the variational \snode{} on a training set with early stopping on a validation set and compare the prediction for the possible covariates with the Aalen--Johansen estimators obtained on a test set.
\vspace{-10pt}

\paragraph{Calibration of the credible intervals}
While our model captures the non-parametric estimator by visual inspection, we seek to quantify the calibration performance in simulations where the ground truth is known. Again using the \textsf{R} package \textbf{\texttt{coxed}}, we simulate a survival data set with three covariates. From the \textbf{\texttt{coxed}} package we also extract the underlying individual survival probabilities. To estimate calibration of the error intervals, we therefore calculate the average of fraction of times the true survival probabilities we sample from lie within the 95\% credible interval. We compare the calibration of our model to the prediction from a Cox proportional hazards model using the \textsf{R} \textbf{\texttt{survival}} package~\citep{survival-package}, which implements the calculation of standard errors. For one random realization of the simulated data we perform a five fold cross validation in \autoref{tab_3}.
We find that our model produces more consistent and better calibrated error intervals than the Cox proportional hazards model.
\vspace{-10pt}

\paragraph{Clustering of the latent space}
An additional useful feature of the latent variable model can be found by inspection of the latent space of the model. We again simulate an illness-death model data-set with \textbf{\texttt{coxed}}, using nine covariates. We again run the variational \snode~model with early stopping using a validation set and then inspect the latent space for the validation data. Using UMAP \citep{mcinnes2018umap-software} we identify five clusters (\autoref{fig:bayes_multistate_cluster_kmeans_km}). We examine the probabilities to be in each of the three states for each cluster in the validation data set using the non-parametric Aalen--Johansen estimator. As can be seen in \autoref{fig:bayes_multistate_cluster_kmeans_km}, the clusters are a meaningful unsupervised differentiation between patients and capture survival differences as well as differences in transitioning to the "Illness" state well. We can additionally obtain covariate effects associated with each cluster by using logistic regression. This feature has useful applications in a clinical setting, where identification of extreme survivors to a treatment while modeling other state transitions is of particular interest. Our approach is directly applicable to survival analysis, where methods for example based on LDA \citep{Chapfuwa_2020} were recently proposed to cluster the latent space, but generalizes those to the multi-state setting.
\vspace{-10pt}

\section{Conclusion}
We have introduced a general and flexible method for multi-state survival analysis based on neural ODEs and shown state-of-the art performance in the special cases of survival with a superior performance for Brier score and binomial log likelihood. In addition, we have demonstrated that \snode{} is capable of accurately recovering the hazard rates of a general multi-state model. Finally, a variational approach allows for the estimation of credible intervals and provides an interpretability aspect by introducing latent states.

\begin{figure}[H]
    \centering
    \subfigure[]{\includegraphics[width=0.45\textwidth]{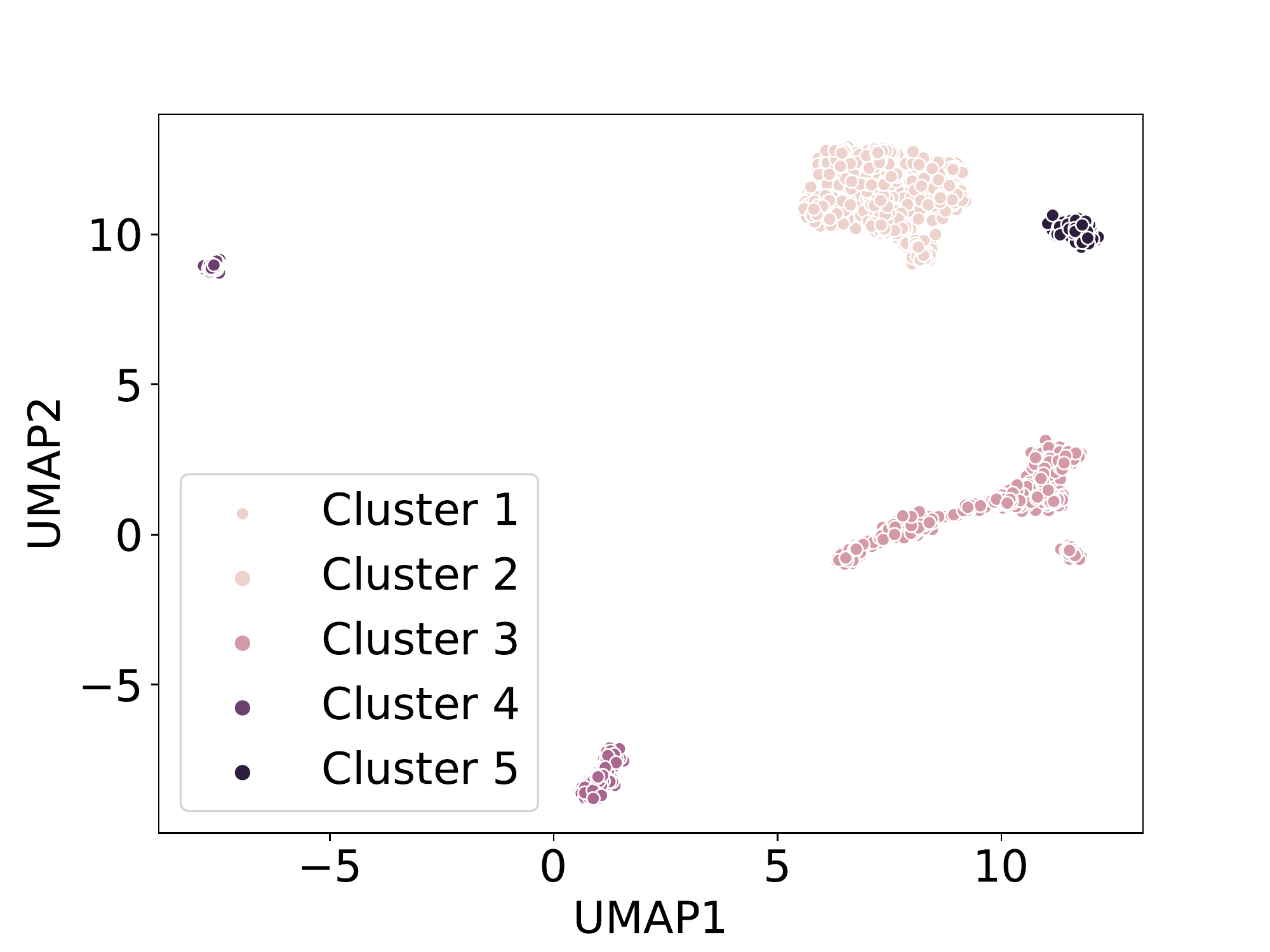}}
    \vspace{-2pt}
    \subfigure[]{\includegraphics[width=0.45\textwidth]{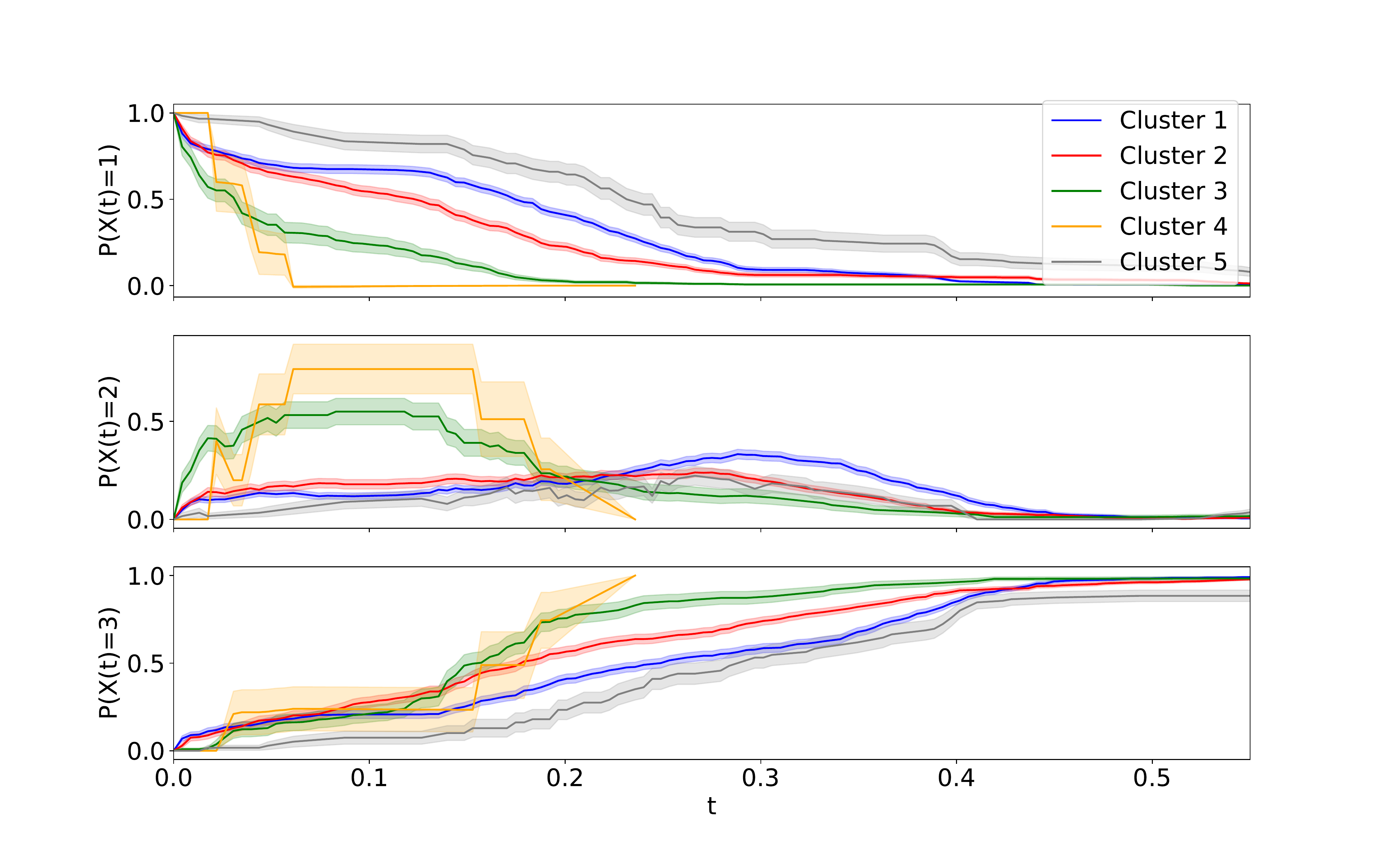}}
    \caption{The latent space of the variational \snode{} model shows meaningful clusters. The subset of patients in the clusters on the left are used in a non-parametric Aalen--Johansen estimator to obtain state occupation probabilities for all three states per cluster, which differ significantly. The first panel is the probability for the "Health" state, the second for the "Illness" and the third for the "Death" state.}
    \label{fig:bayes_multistate_cluster_kmeans_km}
\end{figure}
\clearpage

\bibliography{references}
\bibliographystyle{icml2020}
\clearpage

\onecolumn
\appendix
\section*{Supplementary Material: \textit{Neural ODEs for Multi-State Survival Analysis}}
\section{Proof of the inverse relationship of the Kolmogorov backward equation}
In Algorithm \ref{algo2} we want to cheaply compute $\bm{P}(s, t)$ for any $s, t$ using $\bm{P}^{-1}(0,s)\bm{P}(0,t) = \bm{P}(s,t)$. A na\"ive implementation would require inversion of $\bm{P}(0, s)$ for any $s$ where we evaluate this matrix which can become costly. 
Instead, we leverage the connection of the Kolmogorov forward equation with the respective Kolmogorov \emph{backward} equation 
\[\frac{d\bm{P}(s,0)}{ds} = - \bm{Q}(s) \bm{P}(s, 0).\]
Namely, we use that the matrix $\bm{P}(s, 0)$ obtained from the backward equation equals $\bm{P}(s, 0) = \bm{P}^{-1}(0, s)$ from the forward equations as shown in the following proposition. The proposition is well-known and presented only for reference. This relationship justifies the use of the same letter $\bm{P}$ in both sets of differential equations as well as the description as \enquote{forward} and \enquote{backward} equations. 
\begin{prop}
Let $s>0$ and let $\bm{P}(0, s)$ solve the forward equation. Denote $\bm{P}(s, 0)$ the solution to the backward equations. Then $\bm{P}(s,0) = \bm{P}^{-1}(0,s)$.
\end{prop}
\begin{proof}
Denote $\bm{1}$ the identity matrix of appropriate dimension. We can write
\begin{equation}\label{eq:identityrule}
    0 = \frac{d}{ds} \bm{1} = \frac{d}{ds} \left(\bm{P}(0,s)\bm{P}^{-1}(0,s)\right).
\end{equation}
On the other hand, using the product rule as well as the forward equations, we get
\begin{align*}
     &\frac{d}{ds} \left(\bm{P}(0,s)\bm{P}^{-1}(0,s)\right) \\
     &=  \left(\frac{d}{ds}\bm{P}(0,s)\right)\bm{P}^{-1}(0,s) +   \bm{P}(0,s)\frac{d}{ds}\bm{P}^{-1}(0,s) \\
     & =\bm{P}(0,s)\bm{Q}(s)\bm{P}^{-1}(0,s) + \bm{P}(0,s)\frac{d}{ds}\bm{P}^{-1}(0,s).
\end{align*}
Since the overall system is identically $\bm{0}$ because of \eqref{eq:identityrule}, we can rearrange 
\[\bm{P}(0,s)\frac{d\bm{P}^{-1}(0,s)}{ds} = - \bm{P}(0,s) \bm{Q}(s) \bm{P}^{-1}(0,s).\]
Multiplying with the inverse from the left, we are left with
\[\frac{d\bm{P}^{-1}(0,s)}{ds} = - \bm{Q}(s) \bm{P}^{-1}(0,s)\]
which we identify as the Kolmogorov backward equation.
\end{proof}

\section{Lyapunov style loss term}
In some training cases we observed diverging latent state trajectories, making the training procedure unstable and giving us underflow errors in the differential equation solvers. Analogous to \citet{massaroli2020dissecting}, we introduce a loss term for the latent states, which is related to the estimation of Lyapunov exponents and therefore has the interpretation of keeping the differential equation non-chaotic, therefore regularizing the evolution. This is obtained by taking an $L_2$ loss of the latent states at the maximum time in the mini batch $t_{mmb}$,
\[ L_{\text{Ly}} = \frac{1}{M} \sum_{i} \|m_i(t_{mmb})\|_2^2, \]
where $M$ is the number of latent states in the model. We find this additional loss to improve stability of our algorithm for large number of latent states.

The full loss is then given by
\[
L = -\log\mathcal{L}(\theta; \mathcal{Y}) + \mu L_{\text{Ly}}
\]
with the likelihood $\mathcal{L}(\theta; \mathcal{Y})$ specified in \autoref{sec:multi_state}.
For example, in the case of right-censoring which we consider in this paper, $\mathcal{L}(\theta; \mathcal{Y})$ is given in \autoref{sec:rightcensoring}.

\section{Algorithm}
The \snode{} algorithm is given in Algorithm \ref{alg:snode}.

\begin{algorithm*}[h]
    \caption{Obtain $P_{ij}(s,t)$ and $\lambda_{ij}(t)$ in \snode}
    \label{algo2}
\begin{algorithmic}
\STATE {\bfseries Input:} Covariates $\bm{x}$, time interval $(s,t)$.
\STATE $\bm{m}(0) = f_{\theta}(x)$, $\bm{P}(0,0)=\mathbb{1}$ $\rightarrow$ $s_0=(\bm{P}(0,0),\bm{P}(0,0),\bm{m}(0))$ 
\COMMENT{Get initial values.}
\STATE
\STATE \textbf{Function }{\texttt{KFE\_KBE} }{($\bm{P}(0,t),\bm{P}(t,0),\bm{m}(t),t$)}: 
\COMMENT{Kolmogorov forward and backward equation.}
\STATE $\qquad\lambda_{ij}(t),\ M_i(t) = g_{\phi}(\bm{P}(0,t),\bm{m}(t),\bm{x},t)$ \COMMENT{$\lambda$, $M$ from NN with \texttt{softplus} for $\lambda$.}
\STATE $\qquad\lambda_{ii} = -\sum_{k}\lambda_{ik}$ 
\COMMENT{Enforce constraints.}
\STATE $\qquad\frac{\mathrm{d}P_{ij}(0,t)}{\mathrm{d}t} = \sum_k P_{ik}(0,t) \lambda_{kj}$ 
\COMMENT{Calculate gradient for Kolmogorov forward equation.}
\STATE $\qquad\frac{\mathrm{d}P_{ij}(t,0)}{\mathrm{d}t} = -\sum_k \lambda_{ik} P_{kj}(t,0)$ \COMMENT{Calculate gradient for Kolmogorov backward equation.}
\STATE $\qquad\frac{\mathrm{d}m_i(t)}{\mathrm{d}t} = M_i(t)$ 
\COMMENT{Calculate gradient for augmented evolution.}
\STATE \textbf{return} $\left[\frac{\mathrm{d}P_{ij}(0,t)}{\mathrm{d}t},\frac{\mathrm{d}P_{ij}(t,0)}{\mathrm{d}t},\frac{\mathrm{d}m_i(t)}{\mathrm{d}t}\right]$ 
\COMMENT{return derivatives}
\STATE
\STATE $\bm{P}(0,t),\bm{P}(s,0), \bm{m}(t), \dots = \texttt{ODEsolve}(s_0,\texttt{KFE\_KBE},(0,t),\texttt{save\_at}=\{s,t\})$
\STATE $\lambda_{ij}(t) = g_{\phi}(\bm{P}(0,t),\bm{m}(t),t)$ 
\COMMENT{Get the instantaneous transition rate.}
\STATE $\bm{P}(s,t) = \bm{P}(s,0)\cdot \bm{P}(0,t)$ 
\COMMENT{Use the composability to get $\bm{P}(s,t)$}
\STATE $P_{ij}(s,t)$, $\lambda_{ij}(t)$
\end{algorithmic}
\label{alg:snode}
\end{algorithm*}

\section{Implementation details}
All models are implemented in PyTorch~\citep{pytorch} using the \textbf{\texttt{torchdiffeq}} package~\citep{chen2018neural}. As our example networks are sufficiently small, we use backpropagation through the ODE solver to obtain gradients, however, using the adjoint method is of course possible as well. 
We use the \texttt{dopri5} method for the ODE solver with an absolute and relative tolerance of $10^{-8}$ in the ODE solver. To include the accuracy of the solution as a hyperparameter, we scale the event times to have the maximum value $S$, which we choose to be of $\mathcal{O}(1)$. To specify the non-zero elements of the transition rate matrix, a matrix with 1 indicators for non-zero off-diagonal elements and \texttt{NaN} indicators for all other elements are needed.

For training the model minimizing the negative log-likelihood, the hyperparameters are:
\begin{itemize}[noitemsep]
    \item Number of layers $L_e$ and number of neurons per layer $N_e$ with dropout\citep{srivastava2014dropout} $p_e$ for multilayer perceptron encoding the covariates into memory states;
    \item Number of layers $L_Q$ and number of neurons per layer $N_Q$ for multilayer perceptron modeling $\bm{Q}$;
    \item Number of memory states $M$;
    \item Coefficient of Lyapunov style loss term $\mu$;
    \item Scaling coefficient for event times $S$;
    \item Learning rate $l$ of the Adam optimizer~\citep{kingma2014adam};
    \item Weight decay $w$.
\end{itemize}
For the variational approach minimizing the ELBO, we have the hyperparameters:
\begin{itemize}[noitemsep]
    \item Number of layers $L_p$ and number of neurons per layer $N_p$ with dropout $p_p$ for multilayer perceptron for prior $p(z|x)$;
    \item Number of layers $L_q$ and number of neurons per layer $N_q$ with dropout $p_p$ for multilayer perceptron for variational postierior $q(z|x,t)$;
    \item Number of layers $L_Q$ and number of neurons per layer $N_Q$ for multilayer perceptron modeling $\bm{Q}$;
    \item Number of latent states $M$;
    \item Coefficient of Lyapunov style loss term $\mu$;
    \item ELBO parameter $\beta$
    \item Scaling coefficient for event times $S$;
    \item Learning rate $l$ of the Adam optimizer;
    \item Weight decay $w$,
\end{itemize}
where the ELBO parameter $\beta$ characterizes the relative weight between log-likelihood and Kullback-Leibler divergence, which we set to be $1$ throughout the paper. Closer investigation of the clustering property with respect to this parameter would be of interest.

\section{Experiments}
\subsection{Benchmark in competing risks case}
To show the efficacy of our model in the multi-state setting, the simplest extension of survival models is given by the competing risks model (\autoref{fig:comp_hazard}).
In this setting all possible states the model can transition to are absorbing, and hence there are no intermediate states. 
In this specific multi-state case we can benchmark our model against the standard tools for competing risks analysis: 
The cause-specific Cox models, where a Cox proportional hazards model is fit for each transition taking all other transitions as censored; the Fine--Gray model \citep{fine1999proportional}; as well as DeepHit \citep{lee2018deephit} and DeepHit with an additional loss term to specifically improve concordance, at the cost of worse calibration \citep{kvamme2019time}. 

\paragraph*{Benchmark I: (\textsc{synthetic})}
\citet{lee2018deephit} provide the \textsc{synthetic} data set with two possible outcomes from a simulation. We noticed, however, that around $12.5\%$ of observations have events occurring at $t=0$. This does not make sense from a survival standpoint, as one would include patients into the study who have already experienced the event. The inclusion of a large number of events at $t=0$ favors \texttt{DeepHit}, as it is a probability mass function model, meaning it models the probability mass at every observed time point separately, as opposed to assuming a relatively smooth hazard rate. 
Leaving in the events at $t=0$, \snode{} obtains cause specific concordances of ${0.736\pm 0.006}$ and ${0.739\pm 0.007}$ using five fold cross validation, which is still competitive with DeepHit without a ranking loss, which scores $0.739\pm0.002$ and $0.737\pm0.003$ respectively \citep{lee2018deephit}, whereas the cause specific Cox model and the Fine Gray model have cause specific concordances below $0.6$, as found in \citep{lee2018deephit}. Simply removing the patients with events at $0$ and with manual hyper-parameter tuning for \snode{} and \texttt{hyperopt} \cite{bergstra2013making} optimization for DeepHit on the validation set, we find that \snode{} outperforms DeepHit, scoring $c_1=0.74$ and $c_2=0.72$ for the two cause specific concordances, compared to $c_1=0.73$ and $c_2=0.72$ for DeepHit, as implemented in the pycox package \cite{kvamme2019time}.

\paragraph*{Benchmark II: New Dataset}
To obtain a fair comparison and to avoid pathological events at $t=0$, as well as benchmark calibration using the (integrated) Brier score, we simulate $5000$ patients directly from a Markov jump process with two competing absorbing outcomes, using the Gillespie algorithm with some slight proportional hazards violation. We split into $64\%$ train set, $16\%$ validation set and $20\%$ test set and train both \snode{} and DeepHit, implemented in the pycox package \cite{kvamme2019time} with early stopping, as well as manual and systematic hyper-parameter search for \snode{} and DeepHit respectively on train and validation set. For the cause-specific Cox model and Fine-Gray model, we do not have hyper-parameters and therefore train on the combined train and validation set. The trained models are evaluated on the test set.

We calculate cause specific concordances (c), as well as integrated Brier score (ibs) through the integrated squared distance to the ground truth simulated probability distribution for each cumulative incidence functions. The results are shown in \autoref{tab_2}.

\begin{table*}
  \caption{Benchmark in the competing hazards case. Higher concordance and lower integrated Brier score are better.}
  \label{tab_2}
  \centering
    \begin{tabular}{c|c|c|c|c}
    Model & c cause 1 & c cause 2 & ibs cause 1 & ibs cause 2\\ \hline
    Cause-specific Cox model \cite{cox1972regression}& $0.67$ & $0.72$ & $0.20$ & $0.32$ \\
    Fine-Gray \cite{fine1999proportional} & $0.69$ & $0.71$ & $0.13$ & $0.26$\\
    DeepHit \cite{lee2018deephit} & $0.68$ & $0.71$ & $0.47$ & $1.5$ \\
    \rowcolor{Gray}
    \snode~(this paper) & $\bm{0.70}$ & $\bm{0.73}$ & $\bm{0.09}$ & $\bm{0.13}$
    \end{tabular}
\end{table*}

We note that we use an ad-hoc hyperparameter setting and only performed manual hyperparameter tuning using the validation set for \snode{}, whereas we use \texttt{hyperopt} \cite{bergstra2013making} with $300$ cycles of training to find the optimal hyperparameter setting on the validation set for DeepHit. We see that \snode{} outperforms all other models in terms of concordance and especially integrated Brier score. To further demonstrate the good calibration of \snode{}, we show the cumulative incidence functions for all models, as well as the Brier score over time in Figure~\ref{fig:comprisk}. In both figures we see very good calibration of \snode{} at all times, outperforming all other models. 

\begin{figure}
    \centering
    \includegraphics[width=0.49\textwidth]{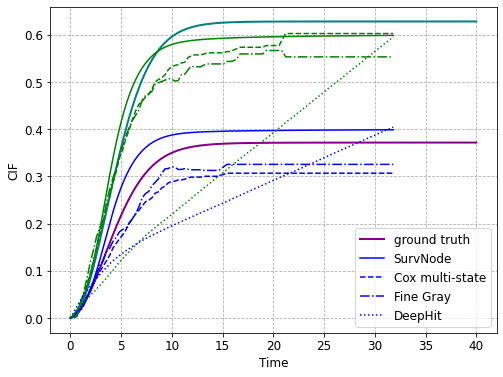}
    \includegraphics[width=0.49\textwidth]{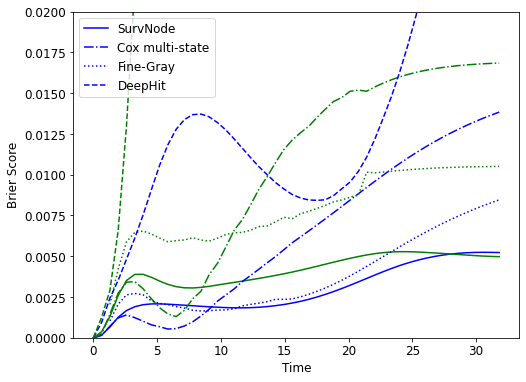}
    \caption{Cumulative incidence functions for the two competing outcomes for all benchmarked models on the top, Brier score for both outcomes and for all benchmarked models on the bottom. Risk $1$ is shown in blue for the estimates and purple for the ground truth, whereas risk $2$ is color coded in green for estimates and teal for ground truth. Lower Brier scores at each time show smaller deviation from the ground truth probability distributions and are therefore better. We can clearly see \snode{} outperforming all other models in terms of Brier score for almost all times.}
    \label{fig:comprisk}
\end{figure}

\subsection{Simulation of data}
The simulated data in the publication is generated in two ways. First, we simulate data with the \textsf{R} package \textbf{\texttt{coxed}}. 

In the survival cases, we choose three covariates, where one of the covariates has time varying coefficients to model a proportional hazards violation. We choose all coefficients to be of $\mathcal{O}(1)$, with a saw-tooth time dependence for the time dependent covariate. We sample $2048$ patients for the training set and $1024$ patients for the validation and test set respectively with event times between $0$ and $100$. In the case of the illness death model, we sample using the \textbf{\texttt{coxed}} package for every transition, assuming independence of each transition. We extract the covariates from the first sampled model and use them for the other two survival realizations, however choosing different coefficients. Due to a limitation of the \textbf{\texttt{coxed}} package, only the first sampled model can have time varying coefficients, with the other transitions then effectively being sampled from a Cox-model. In the competing case between "Illness" and "Death" from the "Health" state, we choose the first occurring time of the two sampled survival data realizations, no matter if there is censoring or not. The maximum time for the generated data in the competing case is $T=100$, whereas we choose $T=50$ for the transition from "Illness" to "Death". 

The second way is to directly sample from a Markov-Jump process. For this we implement a Gillespie sampling algorithm in \textbf{\texttt{Julia}} \cite{bezanson2017julia}, using the \textbf{\texttt{DifferentialEquations.jl}} \cite{rackauckas2017differentialequations} package. We sample parameters for a Weibull distribution for each transition in the multi-state case and multiplicatively add covariate dependence in a proportional hazards way. To break proportional hazards, we use time dependent coefficients for two of the $12$ covariates, as in the above sampling algorithm. We choose all coefficients to be of $\mathcal{O}(1)$. We sample $5000$ patients, which we then split into 64\% training, 16\% validation and 20\% test set. As we specify the underlying hazard functions, ground truth for both hazard functions as well as probability distributions is directly accessible for any multi-state model. The simulation code is available on the \snode{} github page.

\subsection{Data sets and hyperparameters}
The \textsc{metabric} and \textsc{support} data sets are standard survival data sets for benchmarking. The characteristics are shown in \autoref{tab_sp_1}~\citep{kvamme2019time} and are obtained from the \href{https://github.com/havakv/pycox}{{\color{blue} pycox}} python package \citep{kvamme2019time}.
\begin{table*}[t]
  \caption{Characteristics of the \textsc{metabric} and \textsc{support} data sets.}
  \label{tab_sp_1}
  \centering
    \begin{tabular}{c|c|c|c|c}
    Data set &  Size & Covariates & Unique Durations & Prop. Censored \\
    \hline
    \textsc{support} & 8873 & 14 & 1714 & 0.32 \\
    \textsc{metabric} & 1904 & 9 & 1686 & 0.42
    \end{tabular}
\end{table*}
The \textsc{synthetic} data set in the competing hazards case is taken from \citet{lee2018deephit} and available on Github 
with $30000$ patients and two outcomes, where $50\%$ of patients experience any event, whereas the other $50\%$ are censored.

For all benchmark experiments we do a five-fold cross validation where we split the data in an $80-20$ split into $20\%$ test-data and the remaining data again in an $80-20$ split into $64\%$ training data and $16\%$ validation data.

The hyperparameter space used in the benchmarks on \textsc{metabric} and \textsc{support} are
\begin{itemize}[noitemsep]
    \item $L_e = 2$ with $N_e=[400,1000]$ and $p_e=0.1$;
    \item $L_q=[2,4]$ with $N_q=[400,1000]$
    \item $M=[50,150]$;
    \item $\mu=10^{-4}$;
    \item $S=2.$;
    \item $l=[1e-4,1e-3]$;
    \item $w=[1e-7,1e-3]$.
\end{itemize}
We use random sampling from the hyperparameter space to get $16$ realizations of the hyperparameters. The batch size is taken to be either $512$ or the length of the data set, whichever is smaller.

For the competing hazards experiment we use the hyperparameters

\begin{itemize}[noitemsep,topsep=0pt,parsep=0pt,partopsep=0pt]
    \item $L_e = 3$ with $N_e=200$ and $p_e=0.1$;
    \item $L_Q=3$ with $N_Q=800$
    \item $M=20$;
    \item $\mu=10^{-3}$;
    \item $S=1.$;
    \item $l=5e-4$;
    \item $w=1e-3$.
\end{itemize}
For the comparison with the non-parametric Aale--Johansen estimator the hyperparameters used for the model minimizing the negative log likelihood were
\begin{itemize}[noitemsep,topsep=0pt,parsep=0pt,partopsep=0pt]
    \item $L_e = 2$ with $N_e=800$ and $p_e=0.$;
    \item $L_Q=3$ with $N_Q=1000$
    \item $M=20$;
    \item $\mu=10^{-4}$;
    \item $S=1.$;
    \item $l=1e-4$;
    \item $w=1e-7$.
\end{itemize}
In the case of the latent model minimizing the ELBO we used
\begin{itemize}[noitemsep]
    \item $L_p = 2$ with $N_p=400$ and $p_p=0.$;
    \item $L_q = 2$ with $N_p=1000$ and $p_p=0.$;
    \item $L_Q=3$ with $N_Q=1000$
    \item $M=70$;
    \item $\mu=10^{-4}$;
    \item $S=1.$;
    \item $l=1e-4$;
    \item $w=1e-7$;
    \item $\beta=1$,
\end{itemize}
and for clustering the latent space the hyperparameter setting we use is
\begin{itemize}[noitemsep]
    \item $L_p = 2$ with $N_p=400$ and $p_p=0.$;
    \item $L_q = 2$ with $N_p=400$ and $p_p=0.$;
    \item $L_Q=2$ with $N_Q=1000$
    \item $M=50$;
    \item $\mu=10^{-4}$;
    \item $S=1.$;
    \item $l=5e-5$;
    \item $w=1e-7$;
    \item $\beta=1$.
\end{itemize}
Finally, for the Brier score in the simulated Illness-Death model, we use
\begin{itemize}[noitemsep,topsep=0pt,parsep=0pt,partopsep=0pt]
    \item $L_e = 3$ with $N_e=50$ and $p_e=0.1$;
    \item $L_Q=3$ with $N_Q=200$
    \item $M=50$;
    \item $\mu=10^{-5}$;
    \item $S=1.$;
    \item $l=1e-4$;
    \item $w=1e-4$,
\end{itemize}
and in the real world example we use
\begin{itemize}[noitemsep,topsep=0pt,parsep=0pt,partopsep=0pt]
    \item $L_e = 3$ with $N_e=100$ and $p_e=0.1$;
    \item $L_Q=2$ with $N_Q=2000$
    \item $M=50$;
    \item $\mu=10^{-7}$;
    \item $S=1.$;
    \item $l=1e-3$;
    \item $w=1e-8$.
\end{itemize}
all of which were only manually hyperparameter tuned on train and validation set.
\section{Visualisation of calibration in the survival setting}
We can examine the calibration of the model in the simple case of one binary covariate. In this case we can use the population level non-parametric Kaplan--Meier estimator \citep{kaplan1958nonparametric} to obtain the survival function $S(t)$. We use the \textsf{R} package \textbf{\texttt{coxed}} \citep{harden_kropko_2019} to simulate survival data with proportional hazards violation and one binary variable \texttt{var}.
We split the data set into training, validation and test set and obtain the Kaplan--Meier estimator for both variable $\texttt{var}=0$ and $\texttt{var}=1$ on the test data. The survival model is trained on the training data with early stopping using the validation data and predicted for $\texttt{var}=0$ and $\texttt{var}=1$. This prediction is compared to the Kaplan--Meier estimator on the test data. We compare our model (\snode) with a Cox proportional hazards model, a fully parametric accelerated failure time model based on the Weibull distribution \citep{collett2003modelling}, as well as DeepHit \citep{lee2018deephit} and Cox-Time \citep{kvamme2019time}, a discrete and continuous time machine learning model, respectively. The visual comparison can be seen in \autoref{fig:survival_compare_cox}.
\begin{figure}
    \centering
    \includegraphics[width=0.49\textwidth]{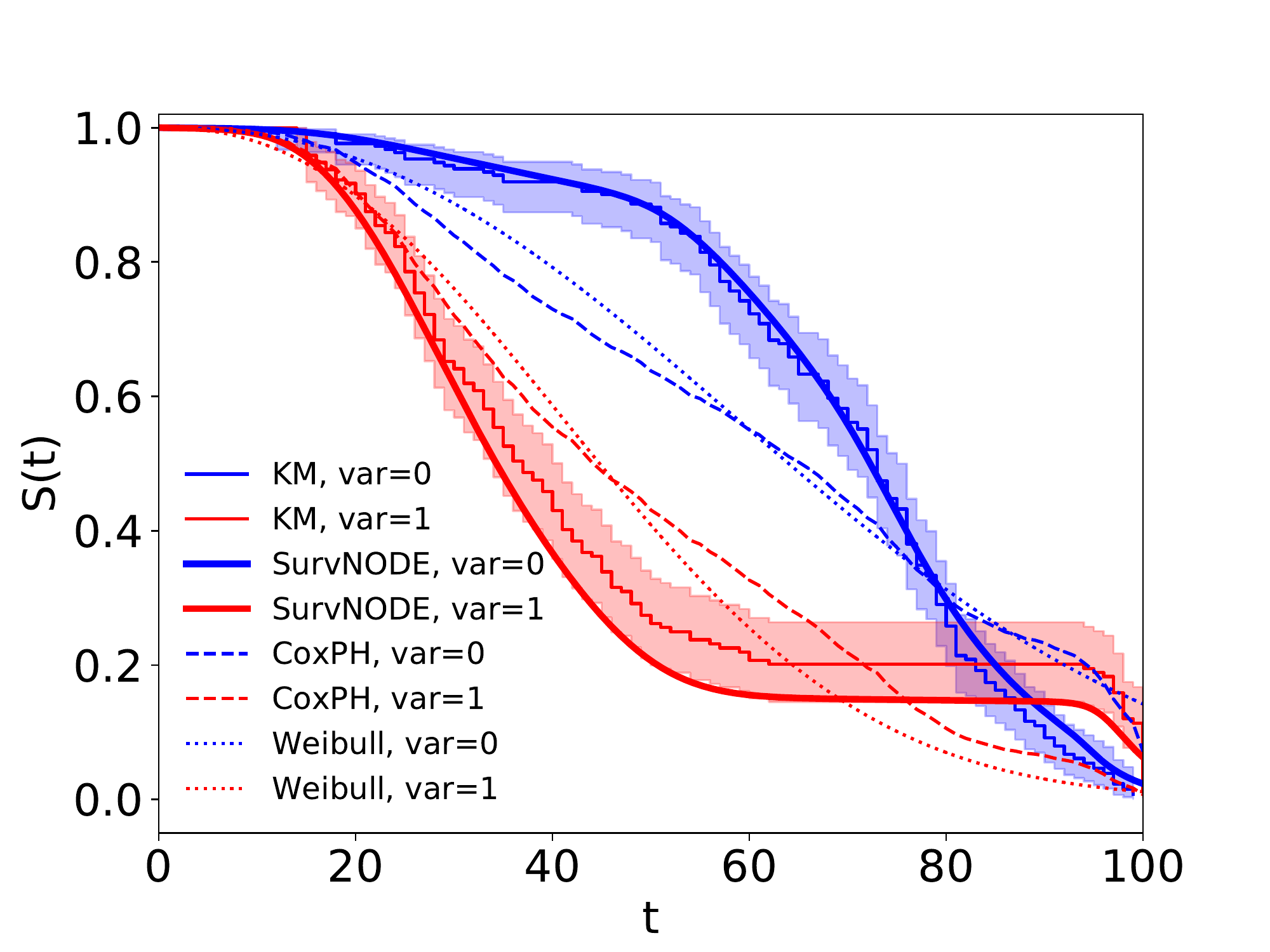}
    \includegraphics[width=0.49\textwidth]{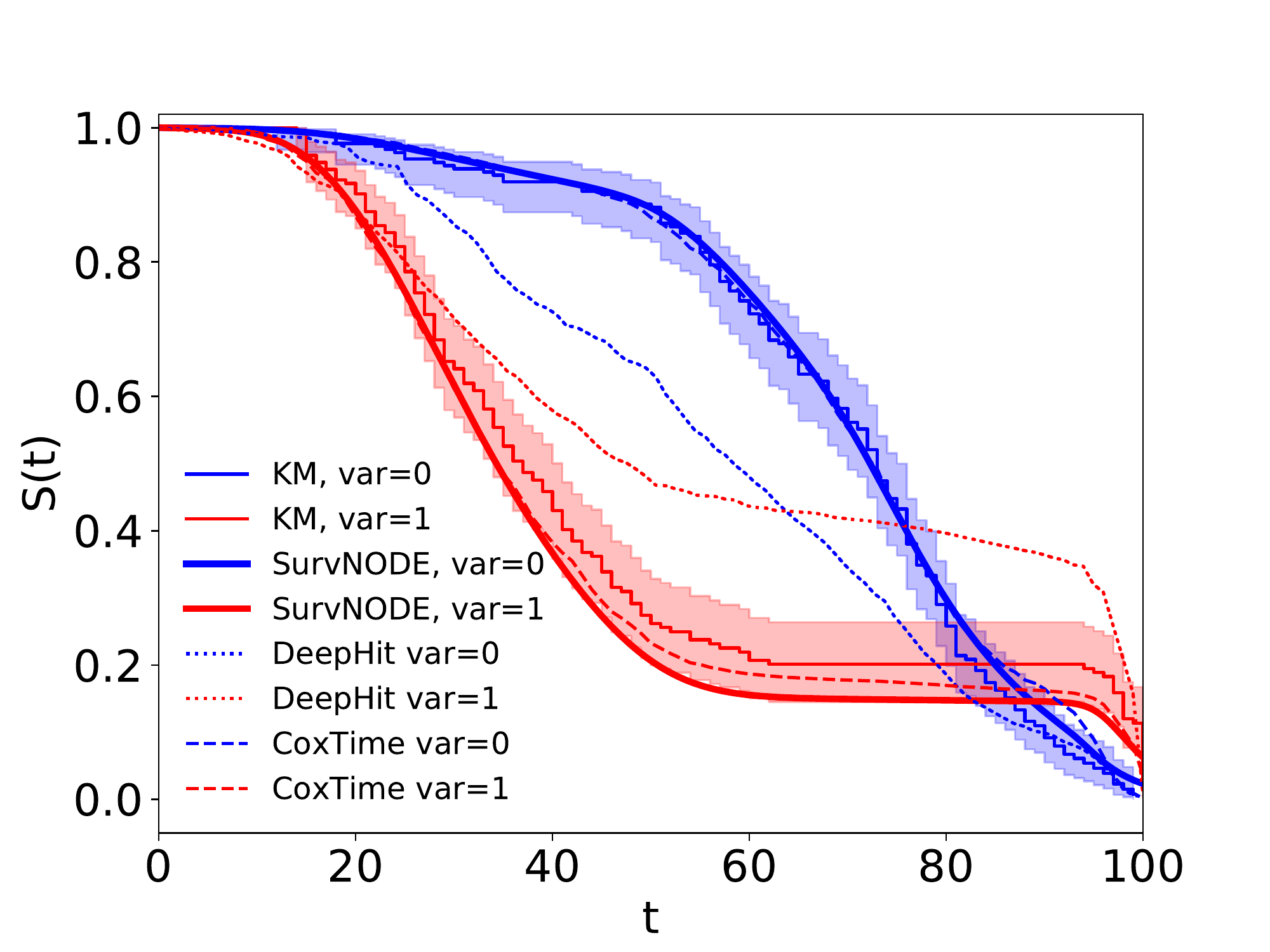}
    \caption{Comparison of different tools to fit survival distributions in the population level case of one binary covariate. On the top our model is compared with the standard tools of survival analysis and the non-parametric Kaplan--Meier estimator of the test set. On the bottom we compare our model on the same data with two state of the art machine learning approaches, again with the Kaplan--Meier as a non-parametric estimator.}
    \label{fig:survival_compare_cox}
\end{figure}
We see that due to the proportional hazard violation, the Cox model as well as the model based on the parametric Weibull distribution do not capture the survival function well, whereas the SurvNODE model does. Comparing to the other machine learning based frameworks, we see that DeepHit does not reproduce the survival function well.
\section{Clustering: Covariates and survival strata}
To further examine the clustering of the latent space, we can superimpose the nine binary covariates in the model on the UMAP projection. This can be seen in \autoref{fig:cluster_covar}. We see that some of the clusters clearly reflect the covariates, for example in the case of covariate one, which is the lowest third of the covariate with the largest effect size for one of the transitions in the simulation, we see that almost all the values are in one of the clusters. By characterizing the effect of the covariates on these clusters with specific survival properties, we can obtain the influence of the covariate on survival.
\begin{figure}
    \centering
    \includegraphics[width=\linewidth]{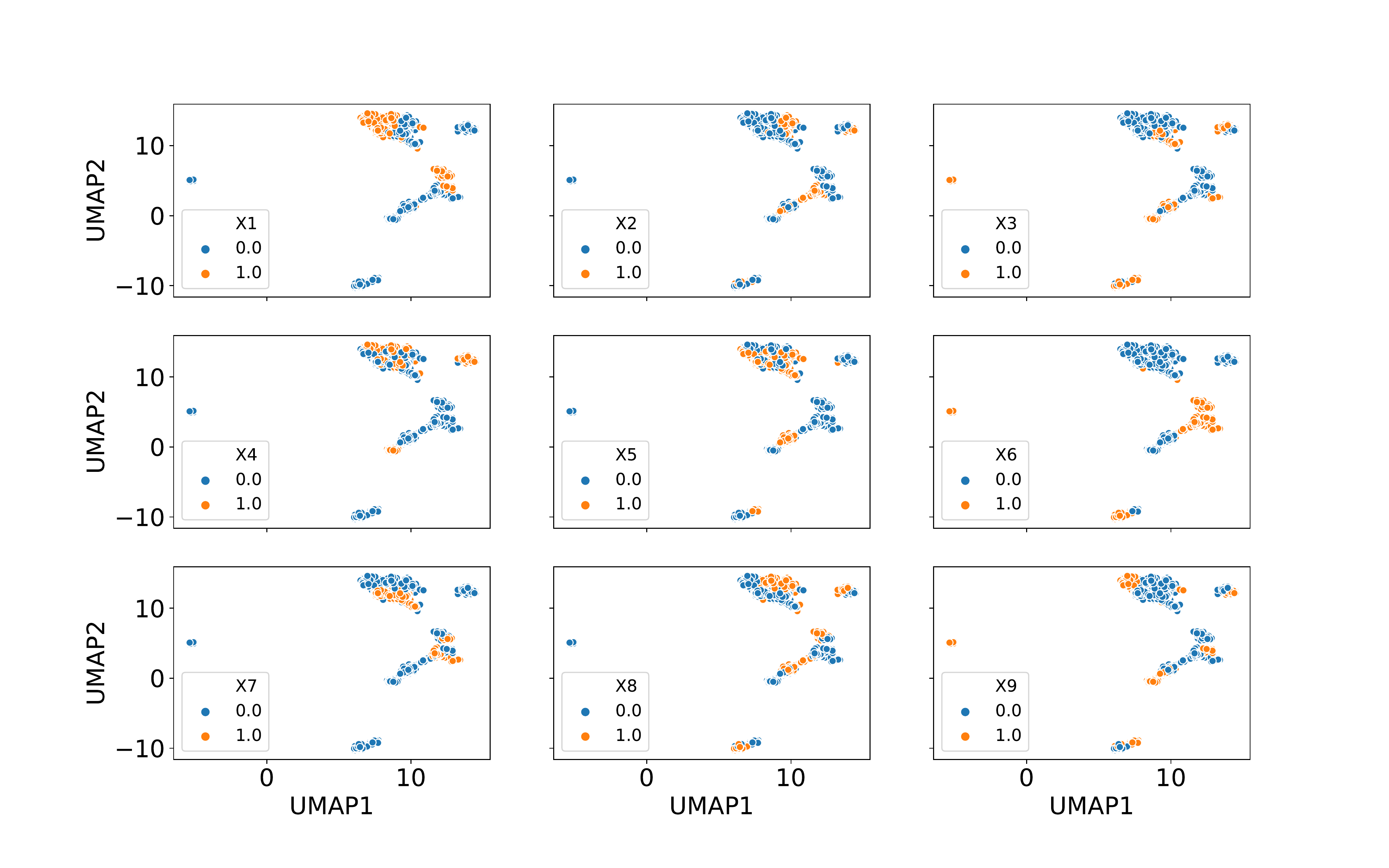}
    \caption{Possible values of the nine binary covariates in the model. We see that some of the clusters clearly reflect the covariates.}
    \label{fig:cluster_covar}
\end{figure}

\begin{figure}
    \centering
    \includegraphics[width=.5\textwidth]{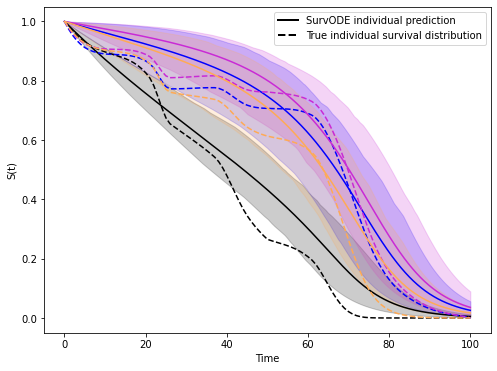}
    \caption{Predicted survival function vs real underlying survival function from the simulation. We see that the credible intervals cover the underlying survival function well.}
    \label{fig:cred_calib}
\end{figure}

\section{Calibration of the credible intervals}
A visual way to show the calibration of the credible intervals is to predict individual survival over time and plot together with the true underlying survival function obtained from the \textbf{\texttt{coxed}} \textsf{R} package. This can be seen in \autoref{fig:cred_calib}. We see that the credible intervals contain the survival function in most of the cases.
\end{document}